\documentclass[12pt]{article}
\usepackage{latexsym,amsmath,amssymb,graphicx}

\newtheorem{thm}{Theorem}[section]
\newtheorem{lemma}[thm]{Lemma}
\newtheorem{propo}[thm]{Proposition}
\newtheorem{coro}[thm]{Corollary}
\newtheorem{fig}{Figure}

\newenvironment{proof}{\begin{trivlist}\item[]{\em Proof.\/\ }}%
                      {\hfill$\Box$ \\ \end{trivlist}}

\title{Space and camera path reconstruction for omni-directional vision}
\author{Oliver Knill and Jose Ramirez-Herran
\footnote{Harvard University, this research was supported by the Harvard Extension School}}
\date{August 17, 2007}

\begin{document}
\maketitle

\abstract{
In this paper, we address the inverse problem of reconstructing a scene as well as the camera 
motion from the image sequence taken by an omni-directional camera.
Our structure from motion results give sharp conditions under which the reconstruction is unique.
For example, if there are three points in general position and three omni-directional 
cameras in general position, a unique reconstruction is possible up to a similarity.
We then look at the reconstruction problem with $m$ cameras and $n$ points, where $n$ and $m$ can 
be large and the over-determined system is solved by least square methods.  
The reconstruction is robust and generalizes to the case of a dynamic environment
where landmarks can move during the movie capture. Possible applications of the result are computer
assisted scene reconstruction, 3D scanning, autonomous robot navigation, medical tomography 
and city reconstructions. 
}

\section{Introduction}

In this paper we address the structure from motion (SFM) problem for omni-directional, 
central panoramic cameras.
The SFM problem is the task of doing a simultaneous reconstruction of objects and camera
positions from the pictures taken by a moving camera. We explore here 
the reconstruction problem for oriented omni-directional cameras, spherical cameras for 
which the reconstruction is particularly convenient. Such cameras can be realized as 
central catadioptric systems which have become so affordable that capturing panoramic 
360 degree images has become popular photographic technique. Omnidirectional vision 
offers a lot of benefits. It is easier to deal with the rotation of the camera for example,
objects do not disappear from view but only change their angular image positions. 
Omnidirectional cameras share the simplicity of orthographic affine cameras 
and have all the benefits of perspective cameras. See \cite{KnillRamirezUllman}.
Unlike orthographic cameras, the camera location is determined. Not at least, the
eye vision of some insects comes close to panoramic vision. It is not surprising that there is already a 
large scientific literature dealing with this part of computer vision. \cite{Benosman} \\

Applications of omni-direction vision are in navigation of autonomous vehicles \cite{Ehlgen}, 
robotics \cite{ThrunBurgardFox,FrancisSpacek,
SpacekBurbridge,Tamimi05,Yagi03,Yagi04}, tracking and motion detection,
simultaneous location and mapping \cite{fastslam}, site modeling, image
sensors for security \cite{IshiguroNCT03,Yagi04} and virtual reality
\cite{MatsumotoII03}. The most recent addendum to google maps 
includes 360 degree panorama pictures embedded into the street
maps. Omnidirectional cameras have captured cities so that users can
move around in a virtual reality environment. Omnidirectional pictures
are used already for building virtual cities \cite{IkeuchiSKS04}.
A fundamental problem in robotics is the simultaneous localization and 
mapping problem, commonly abbreviated as SLAM, and also
known as concurrent mapping and localization CML. Of course, this is just
an other name for the structure from motion problem. 
SLAM Problems \cite{fastslam} arise when robots do not have 
access to a map of the environment, nor know their own position. 
In SLAM, the robot acquires a map of its environment while simultaneously 
localizing itself relative to this map. 
SLAM systems have been developed for different sensor types like cameras.
This problem of structure and observer localization is known under the name
{\bf structure from motion} SFM problem in the computer vision literature  \cite{ThrunBurgardFox,Ortiz}.
A standard SFM approach \cite{Szeliski} assumes perspective cameras
from two or more frames. There is an extensive literature available
presenting the mathematics and practical implementation of the different
techniques \cite{trucco} used in such reconstruction: affine
structure from motion and projective structure from motion \cite{forsyth,
hartley}, motion fields of curves \cite{faugeras96}. 
Finally, reconstruction problems matter in 3D scanning techniques, 
where a camera is moved around an object and the camera position has to be
computed too. If the points in the scene can 
move also, the applications expand to security cameras, reconstruction of 
motion in athletics and team sports or CGI techniques in the motion picture
industry. \\

In section 2 we give a brief historical background of the problem and
mention different applications. In section 3 we discuss various relevant camera models. 
In section 4, we give the reconstruction of omni-directional cameras in two-dimensions, 
where a linear system of equations reveals the relationship between points on the 
photographs and the actual point and camera locations. This system is in general 
over-determined and the reconstruction is done with least square methods.
We give sufficient conditions for a reconstruction to be unique. For example, if $n\geq3$
points and $m \geq 3$ cameras are together not on the union of two lines, then a
unique reconstruction is possible. 
In section 5, the result is extended from two dimensions to three dimensions. 
In section 6, the problem is generalized by allowing the points to move while doing
the reconstruction. Finally, in section 8, we discuss the problem when the orientation of the 
omni-directional camera is not known. While for oriented omni-direction cameras, 
our reconstruction is error free and a linear problem, in practice, the point matching 
produces inaccurate results and require error estimates that we present in section 7. 
In the present paper, we ignore the correspondence problem and assume 
that the projections of $n$ points have been matched across $m$ pictures. 
We tested our algorithm with synthetic data and give numerical measurements 
of the reconstruction error in dependence on the size of the perturbations added to 
the image data. 

\section{The structure from motion problem}

Reconstructing both the space and camera positions from observations is an old problem 
in mathematics and computer science. It is an example of an {\bf inverse problem} in geometry.
It is similar to {\bf tomography} but in general nonlinear. The simultaneous Euclidean 
recovery of shape and camera positions from an image sequences is often called the 
{\bf structure from motion problem} SFM. While sometimes the term is used for the 
problem of reconstructing space with known camera positions or camera parameters 
from known point configurations, the SFM problem reconstruct {\bf both} static points 
and camera positions.  This is the definition used in \cite{Ullman,Kanade} and 
treated in various textbooks like
\cite{hartley} 
for perspective cameras. We only focus on {\bf Euclidean reconstruction},
a reconstruction unique up to a translation and rotation. Except for fixing the coordinate
system, we do not assume to have {\bf ground truth}, known ground control points except for fixing
the origin of the coordinate system. SFM is not to be confused with the concept of 
{\bf motion and structure problem} which is a problem to recover the structure from motion 
fields \cite{faugeras96}. In \cite{KnillRamirezInequality}, we have given the following general
definition of the SFM problem: a camera is a transformation $Q$ on a $d$ dimensional 
manifold $N$ satisfying $Q^2=Q$ which has as an image a lower-dimensional 
surface $S$. Given a manifold $M$ of cameras $Q$ for which all $Q(N)$ are isomorphic to a 
the retinal manifold $S$, the SFM problem asks to reconstruct 
$(P,Q) \in N^n \times M^m$ from the image-data matrix
$\{ Q_i(P_j) \; | \; 1 \leq i \leq n,  1 \leq j \leq m \; \} \in S^{n m}$ modulo a global 
symmetry group $G$ which acts both on $N$ and $M$ leaving the image data invariant:
if $(P,Q)$ and $(P',Q')$ are in the same orbit of $G$, then $Q'(P')=Q(P)$.  \\

The field of image reconstruction is part of {\bf computer vision} and 
also related to {\bf photogrammetry} \cite{photogrammetry}, where the focus is on {\bf accurate} measurements.
In the motion picture industry, reconstructions are used for {\bf 3D scanning} purposes or to render
computer generated images {\bf CGI}. Most scanning and CGI methods often work with known 
camera positions or additional objects are added to calibrate the cameras with 
additional geometric objects. As mentioned above, the 
problem is called {\bf simultaneous localization and mapping problem} in the robotics literature
and is also known as {\bf concurrent mapping and localization}. \\

We know from daily experience that we can work
out the shape and position of the visible objects as well as our own position and
direction while walking through our surroundings. Objects closer to us move faster on the retinal surface,
objects far away do less. It is an interesting problem how much and by which way we
can use this information to reconstruct our position and surroundings \cite{Qian,Richards}.
Even with moving objects, we can estimate precisely the position and speed of objects.
For example, we are able to predict the trajectory of a ball thrown to us and catch it.  \\

The mathematical problem of reconstructing of our
surroundings from observations can be considered as one of the oldest tasks 
in science at all because it is part of an ancient {\bf astronomical quest}:
the problem of finding the positions and motion of the planets
when observing their motion on the sky. The earth is the omni-directional camera
moving through space. The task is to compute
the positions of the planets and sun as well as the path of the earth which is the camera. 
This historical case illustrates the struggle with the structure from motion problem: there was an evolution 
of understanding from Aristoteles, the Ptolemaic geocentric model over the Copernican heliocentric system 
to the discoveries of Brahe, Kepler and Newton. \\

An other seed of interest in the problem is the two dimensional problem of {\bf nautical surveying}. 
A ship which does not know its position but its orientation measures the angles between various
points it can see. It makes several observations and observes cost points. The task is to draw a map 
of the coast as well as to reconstruct the position of the ship. \cite{beautemps}.\\

We develop here a fresh and elementary approach for computing and reconstructing panoramic              
three dimensional scenes from omni-directional video sequences. Similar than other
techniques, we reduce the reconstruction to a least square problem and obtain the unknown
structure as well as the camera path from the image sequence. While the equations
are nonlinear, they can be reduced to linear problems. In comparison, the SFM problem for 
affine orthographic cameras are nonlinear \cite{KnillRamirezUllman}. 
While our approach is simple, it is flexible and and generalizes when 
the objects in the scene are allowed to move. In the case of a static scene recorded without errors, the 
algorithm reconstructs the observed points exactly. It is not an approximation. Uniqueness of 
the reconstruction is assured under mild non-collinearity conditions.  
One of the goals in this paper to point out such borderline ambiguities.  Due to the linearity of the 
problem, the ambiguities appear on linear subspaces of the full configuration space and can be analyzed
with elementary geometric methods.   \\

The mathematics of the structure from motion problem has a rich history. We mentioned astronomy and nautical
surveying, but there are origins in pure geometry as well: from Euclid's work on optics, to Chasles,
Helmholtz and Gibson \cite{Koenderink}. For perspective cameras, the reconstruction of camera and points
from 7 point correspondences and two cameras has been addressed by Chasles in 1855 from a purely 
mathematical point of view \cite{Chasles}.  \\ 
One of the first mathematical results in the structure from motion problem beyond the stereo situation
is {\bf Ullman's theorem} from 1979, which deals with orthographic projections.
"For rigid transformations, a unique metrical reconstruction is 
known to be possible from three orthogra
phic views of four points" \cite{Ullman}. 
Ullman's theorem deals with orthographic affine cameras, cameras for which
the camera center is at infinity. Modulo a reflection, it is possible to recover the point 
positions as well as the planes from the projections in general for four points and three cameras. 
We have given explicit {\bf locally unique} 
reconstruction formulas for 3 cameras and 3 points in \cite{KnillRamirezUllman}.  \\

We show here that for omni-directional vision with fixed orientation, 3 points and three cameras allow
a unique reconstruction if the 6 points are not contained in two lines and both cameras and point 
configurations are not collinear. 
If we have two oriented omni-directional cameras and two points, a reconstruction is possible 
uniquely if and only if the four points are not collinear. The mathematics for 
omni-directional cameras which are not oriented is more complicated because the equations become
transcendental. For omni-direction vision without orientation, 3 cameras and 3 points are enough in general.
Ullman's theorem in the affine case actually can be considered to be a limiting case of an omni-directional result
when all camera centers go to infinity.  \\

For previous approaches to this problem ranging from the
classical stereo vision methods which uses only to frames to the 
$n$-views vision, see \cite{faugeras,hartley,faugeras96,szeliski94recovering}. 
This is an interesting and active area of research on computer vision.

\section{Spherical cameras}
A {\bf camera} in space is a smooth map $Q$ from three dimensional space to a $2$-dimensional
{\bf retinal surface} $S$ so that $Q^2=Q$ \cite{KnillRamirezInequality}.
Of particular interest are cameras, where the surface $S$ is a sphere: \\

A {\bf spherical camera} $Q$ in space is defined by a point $C=C(Q)$ and a sphere $S=S(Q)$ centered
at $C$. The camera maps $P$ to a point $p=Q(P)$ on $S$ by intersecting the line $CP$
with $S$. We label a point $p$ with two spherical Euler angles $(\theta,\phi)$. We also use the more 
common name {\bf omni-directional cameras} or {\bf central panoramic cameras}. Of course, the radius of the 
sphere does not matter. A point
$P$ is seen by the camera by the spherical data  $(\theta,\phi)$.
In two dimensions, one can consider {\bf circular camera} defined by a point $C$ and a circle $O$ around the
point. A point $P$ in the plane is mapped onto a point $p$ on $O$ by intersecting the line $CP$
with $O$. Spherical and circular camera only have the point $C$ and the 
orientation as internal parameters. The radius of the sphere is irrelevant. 
One could also look at {\bf cylindrical cameras} in space is defined by a point $C$ and a cylinder $C$ with axes $L$.
A point $P$ is mapped to the point $p$ on $C$ which is the intersection of the linesegment
$CP$ with $C$. A point $p$ in the film can be describe with cylinder coordinates $(\theta,z)$. 
Because cylindrical cameras capture the entire world except for points on the 
symmetry axes of the cylinder, one could include them in the class of {\bf omni-directional cameras}. 
Omnidirectional camera pictures are also called {\bf panoramas}, even if only part of the
360 field of view is seen and part of the height are known \cite{Peleg,Benosman}.  \\
Of course, cylindrical and spherical cameras are closely related. Given the height angle 
$\phi$ between the line $CP$ and the horizontal plane and the radius $r$ of the cylinder, 
we get the height $z=r \sin(\phi)$, so that a simple change of the coordinate system matches
one situation with the other. We can also model a perspective camera with  
omni-directional camera pictures: if only a small part of the sphere is taken, the picture is similar
than the projective picture taken by the tangent tangent plane. 
Not at least because spherical cameras do not have a focal parameter $f$ as perspective cameras, 
they are easier to work with.  \\

We say, a spherical camera is {\bf oriented}, if its direction is known. Oriented spherical cameras
have only the center of the camera as their internal parameter. The camera parameter manifold $M$ is therefore 
$d$-dimensional. For non-oriented spherical cameras, there are additionally $d(d-1)/2 = {\rm dim}(SO_d)$ 
parameters needed to fix the orientation of the camera. For $d=2$, this is one rotation parameter, for 
$d=3$, there are three Euler rotation parameters.  \\

Practical implementations of omni-directional cameras are the Sony {\bf "Full-Circle 360"} Lense Mechanisms,
which allows a vertical field of view $-17^{\circ} \leq \phi \leq 70^{\circ}$, 
the IPIX {\bf fish eye lens} which gives $0 \leq \theta \leq 185, -92 \leq \phi \leq 92$ and 
needs two clicks, Microsofts {\bf ring cam}, which combines 4 webcameras to get one 360 panoramas, 
a {\bf "HyperOmniVision"} camera, which is made of a hyperbolic mirror used for mobile robots
\cite{Yagi03}, and the {\bf "0-360 Panoramic Optic"} camera which allows a one click 
vertical field of view of $-62.5^{\circ} \leq \phi \leq 52.5^{\circ}$. This camera is an 
example of a so called {\bf central catadioptric system}, a camera in which mirrors are involved. 
More panoramic cameras, from fish-eye cameras to swing-lense cameras are described in \cite{Solomon}.  
"One click" solutions have the advantage that one can also do 360 movies with one camera, that no stitching
is required and that the picture is taken at the same time. 
For more information on low-cost omni-directional cameras, see \cite{Ishiguro} in \cite{Benosman}. \\

In practice, an omni-directional camera can be considered oriented if an arrow of gravity and 
the north direction vector are both known. A robot on earth with a spherical camera is
oriented if it has a compass built in. It could also orient itself with some reference
points at infinity. We discuss in a later section how one can recover the orientation from the 
camera frames. \\

For an oriented omni-directional camera, we only 
need to know the position so that the dimension of the internal camera space is $f=d$. 
For a non-oriented omni-directional camera, we need to know additionally the orientation which lead to $f=d+d (d-1)/2$ parameters.
In three dimensions, non-oriented omni-cameras match the simplicity of affine orthographic
cameras. An important advantage for the structure of motion problem is that
omni-directional cameras have a definite location. They can model perspective cameras without 
sharing their complexity. \\

We know that in order for one to recover all the point and camera
parameters, the structure from motion inequality 
$$\label{dimensionformula} dn + fm + h \leq (d-1) n m + g  $$
has to be satisfied, where $f$ is the dimension of the internal camera parameter space, 
$h$ is the dimension of global parameters which apply to all cameras and where $g$ is the 
dimension of the camera symmetry group $G$. See \cite{KnillRamirezInequality}. 

\section{Planar omni-directional cameras}

We now solve the reconstruction problem for oriented omni-directional cameras in the plane.
This two-dimensional reconstruction will be an integral part of the general 
three-dimensional reconstruction for oriented omni-directional cameras. It turns out 
that for the omni-directional inverse problem with oriented cameras, the uniqueness of 
the reconstruction in space is already determined by the uniqueness in the plane, because 
if the first two coordinates of all points are known, then the height coordinate is 
determined uniquely by the slopes up to a global translation. 
How many points and cameras do we need?

\begin{center}
\parbox{15.5cm}{
\parbox{7.5cm}{\scalebox{0.45}{\includegraphics{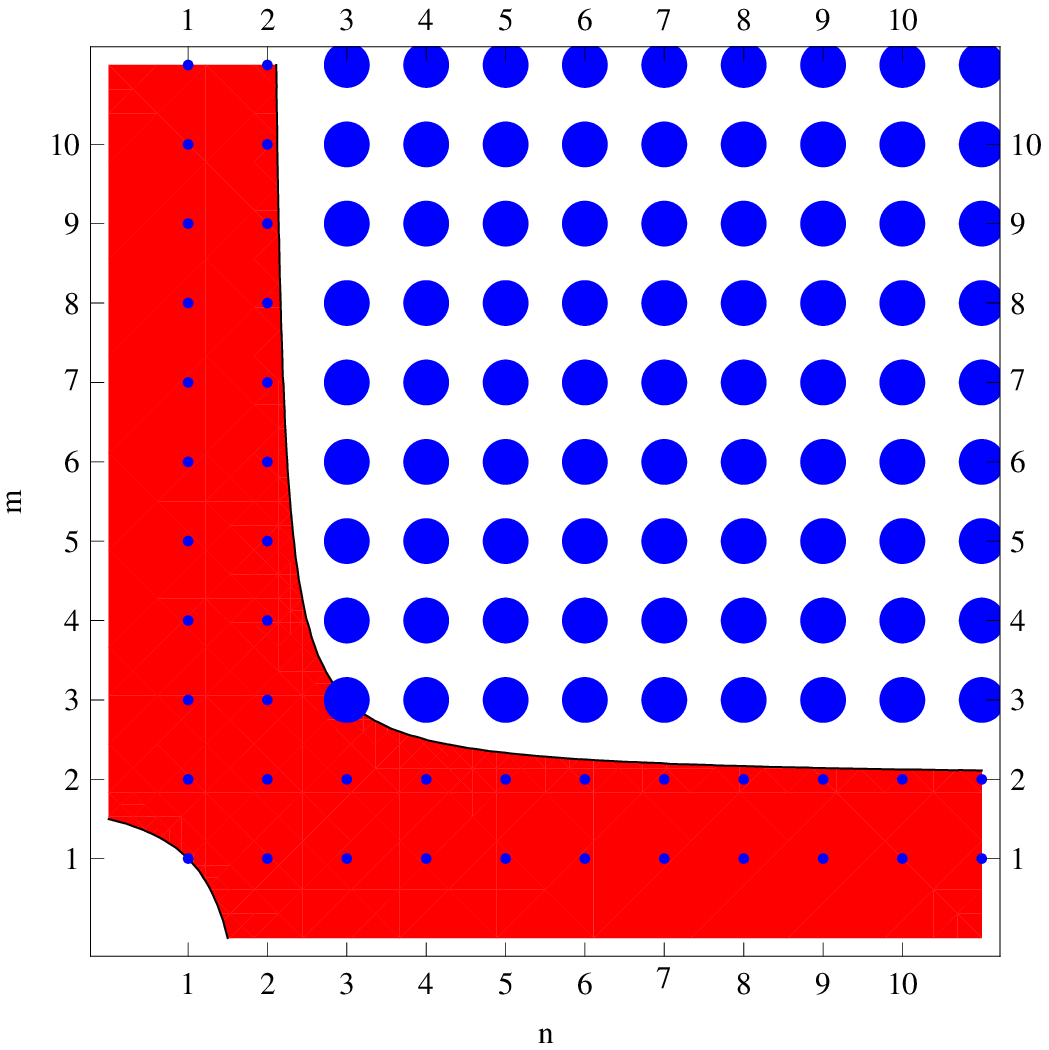}}}
\parbox{7.5cm}{\scalebox{0.45}{\includegraphics{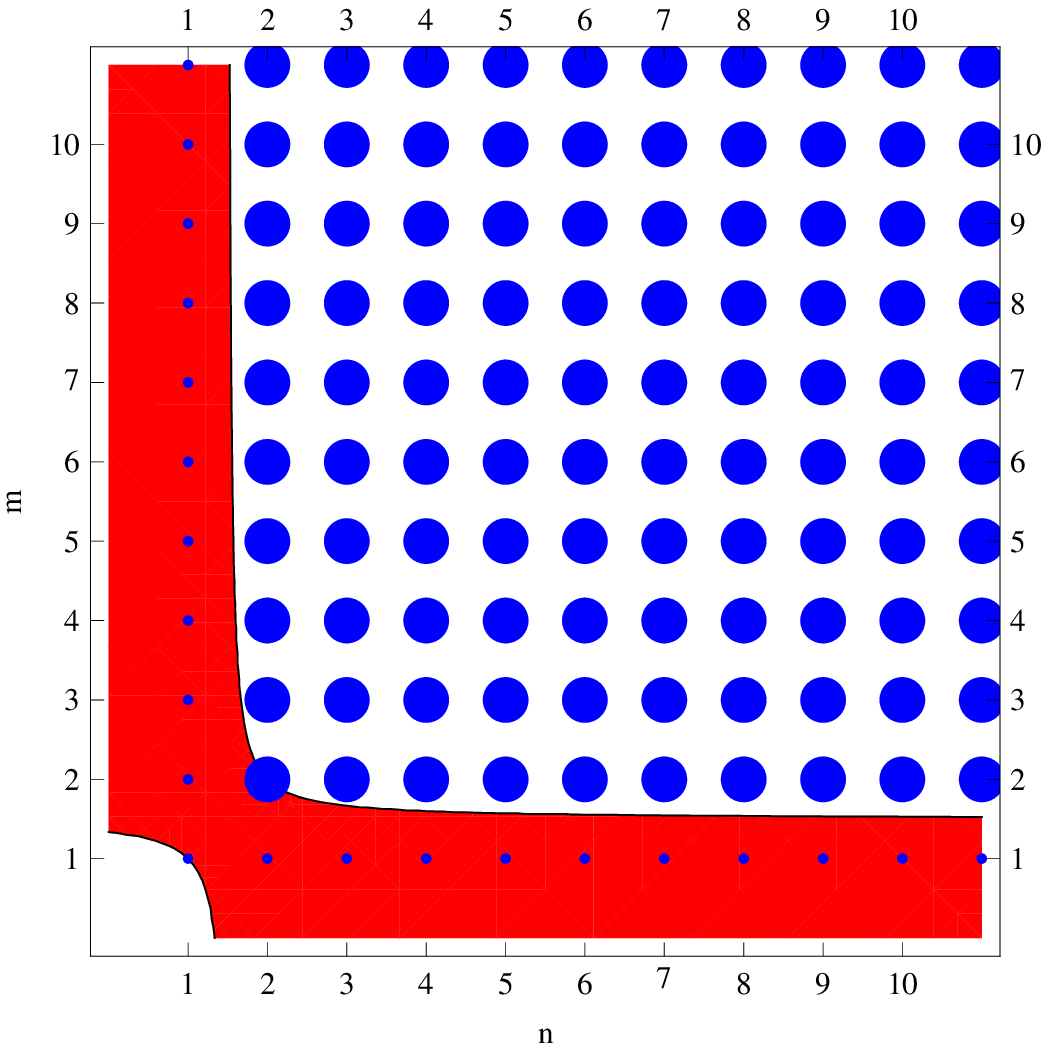}}}
}
\end{center}
\begin{center}
\parbox{15.5cm}{
\parbox{7.5cm}{Oriented Omni  \\ (d,f,g) = (2,2,3)}
\parbox{7.5cm}{Oriented Omni  \\ (d,f,g) = (3,3,4)}
}
\end{center}
\begin{fig}
The forbidden region in the $(n,m)$ plane for oriented omni-directional cameras.
In the plane, $(m,n)=(3,3)$ is a border line case. In space, $(m,n)=(2,2)$ is a
border line case. For $(m,n)$ outside the forbidden region, the reconstruction 
problem is over-determined. 
\end{fig}

Given $n$ points $P_1, \dots ,P_n$ in the plane and an omni-directional camera which
moves on a path $r(t)=(a(t),b(t))$ so that we have cameras $Q_j$ at the points $(a(t_j),b(t_j))$.
We assume that the camera has a fixed orientation in 
the sense that a fixed direction of the camera points north at all times. 
The camera observes the angles $\theta_{i}(t_j)$ under which the points are seen. The angles
are defined if we assume that the camera path is disjoint from the points $P_i$.  \\

\parbox{6.5cm}{\scalebox{0.60}{\includegraphics{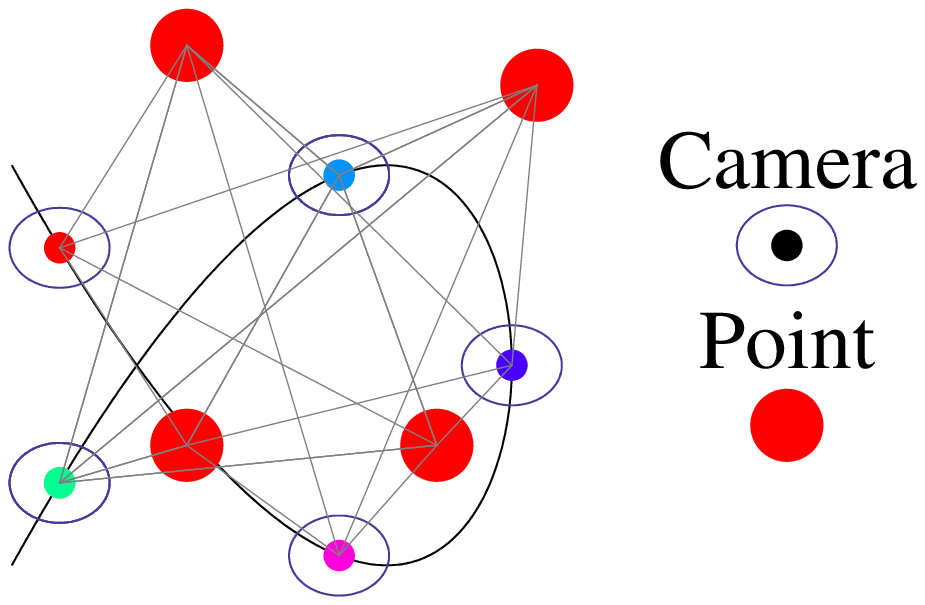}}}
\begin{fig}
The structure from motion problem for omni-directional cameras in the plane. 
We know the angles between points and cameras and want to reconstruct both the 
camera positions as well as the point positions up to a global similarity.
\end{fig}

It is a standing assumption in this article that two different cameras are at two different 
locations, two different points are at two different locations and no camera is at the same
place than a point. In other words, we always assume to deal with $n+m$ different points.  \\

How do we reconstruct the camera positions $Q_j = r(t_j)$ and the points $P_1, \dots ,P_n$
from the angles, under which the cameras see the points?  \\

If $P_i=(x_i,y_i)$ are the $n$ points and $Q_j = (a(t_j),b(t_j))$ are the $m$
camera positions, we know the slopes $\sin(\theta_{ij})/\cos(\theta_{ij})$ if $P_i \neq Q_j$. 
The linear system of $n m$ equations
$$  \sin(\theta_{ij}) (b_i - y_j) = \cos(\theta_{ij}) (a_i - x_j)      $$
for the $2n$ variables $a_i,b_i$ and $2m$ variables $x_j,y_j$ allows in general a reconstruction
if $m n \geq 2n+2m$. But the reconstruction is not unique: the system of equations is still homogeneous
because scaling and translating of a solution produces a new solution. 
By fixing one point $x_1=y_1=0$, the translational symmetry
of the problem is broken. By fixing $x_2=1$ or the distance between $P_1$ and $P_2$ 
the scale is fixed. So, if $m n \geq 2n+2m-3$,
we expect a unique solution. If we write the system of linear equations as $A x = b$, 
then the least square solution is $x = (A^T A)^{-1} A^T b$. \\

For example, in the plane, for $n=3$ points and $m=3$ cameras, we can already reconstruct both 
the point and camera positions in the plane in general: there are 
$2n+3m=(2 \cdot 3) + (2 \cdot 3) =12$ unknowns and $9$ equations. The similarity invariance
fixes $3$ variables so that we have the same number of equations than unknowns.   \\

It is important to know when the reconstruction is unique and if the system is overdetermined,
when the least square solution is unique. In a borderline case, the matrix $A$ is a square
matrix and uniqueness is equivalent to the invertibility of $A$. In the overdetermined case, 
we have a linear system $Ax=b$. There is a unique least square solution
if and only if the matrix $A$ has a trivial kernel. \\

We call a point-camera configuration {\bf ambiguous}, if there exists more than one
solution of the inverse problem. The point-camera configuration is ambiguous if and 
only if the matrix $A$ has a nontrivial kernel. \\

For ambiguous configurations, the solution space 
to the reconstruction is a linear space of positive dimension. Examples of 
an ambiguous configuration are {\bf collinear configurations}, where all points as 
well as the camera path lie on one line. In that case, the points seen on the 
image frames are constant. One can not reconstruct the points nor the camera
positions. \\

If a scale or origin is not specificied, then $\lambda P_i, \lambda Q_j$ 
would produce a family of solutions with the same angular data. We assume to have
factored out these symmetries and do not call this an ambiguity. \\

We now formulate a fundamental result of circular camera reconstructions in the 
plane. It gives an answer when $m=3$ cameras which have taken pictures of 
$n=3$ points, both the camera and the point positions can be obtained 
uniquely up to a similarity. \\

\begin{thm}[Structure from motion for omni cameras in the plane I\\] 
If both the camera positions as well and the point positions are not collinear
and the union of camera and point positions are not contained in the union of two lines, then 
the camera pictures uniquely determine the circular camera positions together with the 
point locations up to a scale and a translation. 
\end{thm}

Even so the actual reconstruction is a problem in linear algebra, 
this elementary result is of pure planimetric nature: we have two non-collinear 
point sets $\cal{P},\cal{Q}$  whose union is not in the union of two lines, then the angles between 
points in $\cal{P}$ and $\cal{Q}$ determine the points $\cal{P},\cal{Q}$ up to scale and translation.
The result should be seen with the background of ambiguity results in 
the plane like Chasles theorem \cite{hartley} 

We call a point or a camera {\bf stationary} if it can not be deformed without changing the
angles between cameras and points. We call it {\bf deformable} if it can be deformed without
changing angles between cameras and points. If a point or a camera is deformable, it can move
on a line. The reason is that the actual reconstruction problem can be written as a system 
of linear equations. We call a choice of a one-dimensional deformation space the {\bf deformation line} 
of the point or the camera. If we have an ambiguous camera-point configuration, then 
there exists at least one deformable point or camera. The deformation space is a linear space. 

\begin{lemma}[Triangularization]
a) If three non-collinear points $P,Q,R$ are fixed, then each camera $C$ position is determined
uniquely from the camera-to-point angles. \\
b) If three cameras $A,B,C$ are fixed, then the camera-to-point 
angles determine each point in the plane uniquely.
\end{lemma}
\begin{proof}
a) If $C$ is not on the line $PQ$, we know two angles and the length of one side of the triangle
$PQC$. Similarly for the other lines $QR,PR$. Because the intersection of the three lines is empty,
every point $C$ is determined. \\
b) Part b) has the same proof. Just switch $P,Q,R$ and $A,B,C$. 
\end{proof}

\begin{lemma}[Deformation]
a) Every stationary camera must be on the deformation line of a deformable point. \\
b) Every stationary point must be on the deformation line of a deformable camera. \\
\end{lemma}
\begin{proof}
In both cases, the angles would change if the point would not be on the deformation line.
\end{proof}

Now the proof of the theorem. 

\begin{proof} 
By fixing one point $P_1$ and the distance $d(P_1,P_2)=1$ between two points,
the scale and translational symmetry is taken care of. Because the point $P_2$ 
moves linearly and has to stay within a fixed distance, it is fixed too. \\

The two stationary points $P_1,P_2$ define a line $L$. 
By the assumption that the points are not collinear, there exists a third point $P_3$ away from that line.  \\

If this point away from the line $L$ were stationary, we would at least three stationary points
$P_1,P_2,P_3$ which are not collinear and by the triangularization lemma, every camera had to be fixed
and again by the triangularization lemma, every point had to be stationary. \\

So, there is a point $P_3$ away from the line $L$ which can be deformed
without changing the angles. Let's call $M$ the deformation line of $P_3$. 

By the deformation lemma, 
stationary cameras are on $M$, deformable cameras are on $L$. 
Because the set of cameras is not collinear, not all cameras can be on $M$ and 
there exists at least one camera $Q_1$ on $L$.
By the two-line assumption, there exists either a camera or a point away from the two lines. It can 
not be a stationary point $P_4$ because this would give us three fixed points
which would fix the scene by the triangulation lemma. 

If it point $P_4$ is deformable, it defines an other deformation line $K$. Consider the triangle
$Q_1,P_3,P_4$. All its edges deform but the angles stay the same. By Desargues theorem, the 
three lines $K,L,M$ go through a common point. Any fixed camera must now be on this intersection
point and every other camera must be deformable on $L$. But this violates the collinearity
assumption for cameras. 
\end{proof}

{\bf Remark:} Alternatively, we could have fixed the coordinates $x_2=1$ of the second point $P_2$ instead 
of the distance. In that case, we additionally have the possibility that the point $P_2$
deforms on the line $x=x_2=1$. But then, every camera must deform on the line $x=x_1=0$. This 
violates the non-collinearity assumption for the cameras.  \\

The following result assumes less and achieves less. When we assume to see 4 points, the two 
line ambiguities are no more possible and the only condition to avoid is collinearity in the point
set or the camera point set: 

\begin{thm}[Structure from Motion for omni-cameras in the plane II]
If there are 4 points for which not more than
2 are collinear and 3 or more oriented cameras, then a unique reconstruction 
of camera and points is possible up to translation and scale. 
\end{thm}

\begin{proof}
Assume we have 4 points for which not more than
2 are collinear and 3 cameras. It is impossible that three points are stationary, because otherwise, by the lemma
part a) also the cameras had to be stationary and by the lemma part b), all points had to be fixed.     
So, we must have that the two points $P_1,P_2$ are stationary and two points $P,R$ are moving.    
The deformation of $P(t),R(t)$ define lines. Every stationary camera has to be on the intersection of 
these two lines and every non-stationary camera has to be on the line through $P_1,P_2$. Let's call the stationary 
camera $C$ and let the other cameras $A,B$ move on the line $P_1,P_2$. In order that the triangles 
$ABP$ and $ABR$ stay similar, the lines $AB$ and the deformation lines through $P$ and $Q$ have to
go through the common point $C$. This contradicts the assumption that no three cameras are collinear.

\parbox{14.8cm}{
\parbox{6.5cm}{\scalebox{0.60}{\includegraphics{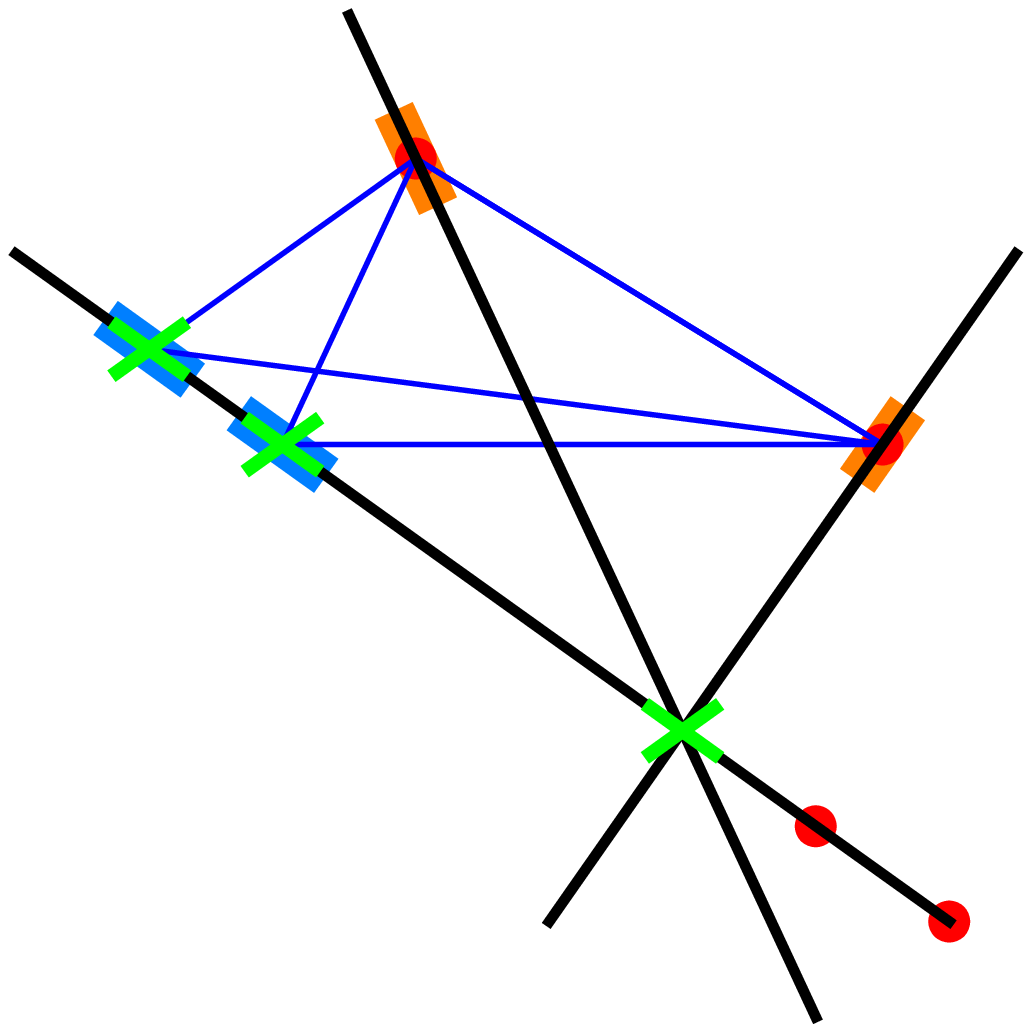}}}
\parbox{6.5cm}{ \begin{fig} 
                To the proof: if two points deform and two cameras deform, their
                deformation lines have to go through a common point by Desargues theorem \cite{berger}
                applied in the special case, when the axis of perspective is the line at infinity. 
                \end{fig}}
}

\end{proof}

The following examples show that we can not relax the conditions in the theorems.
With 3 collinear cameras and 2 or more points, there are families of camera positions with
the same angles. Also, with three non-collinear points and two camera positions, we have 7 unknowns
and 6 equations and can deform the situation. From the dimension formula we
need $nm-2m-2n-3 \geq 0$ to have enough equations. 

\parbox{14.8cm}{
\parbox{6.5cm}{\scalebox{0.60}{\includegraphics{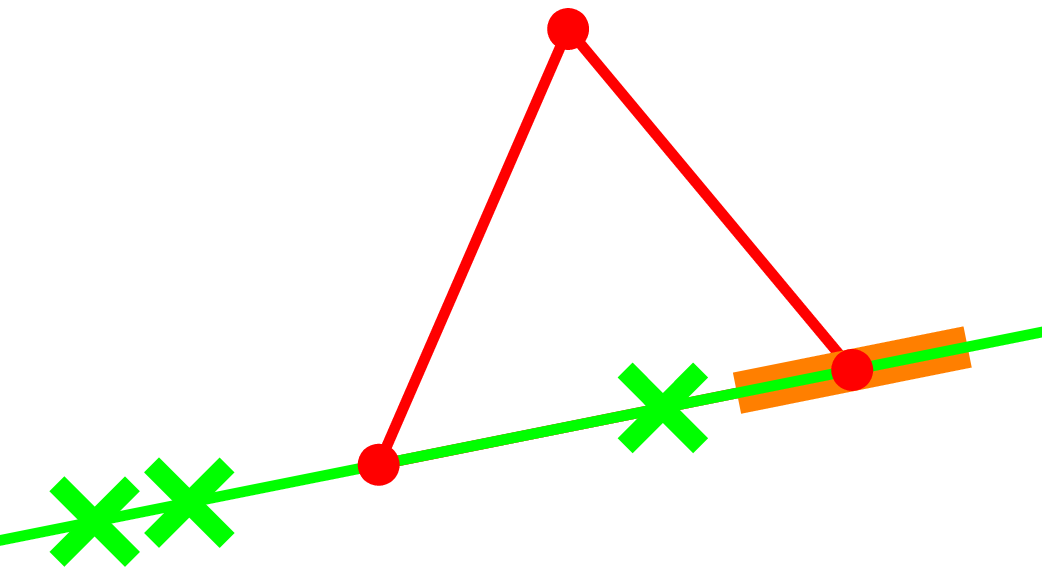}}}
\parbox{6.5cm}{\scalebox{0.60}{\includegraphics{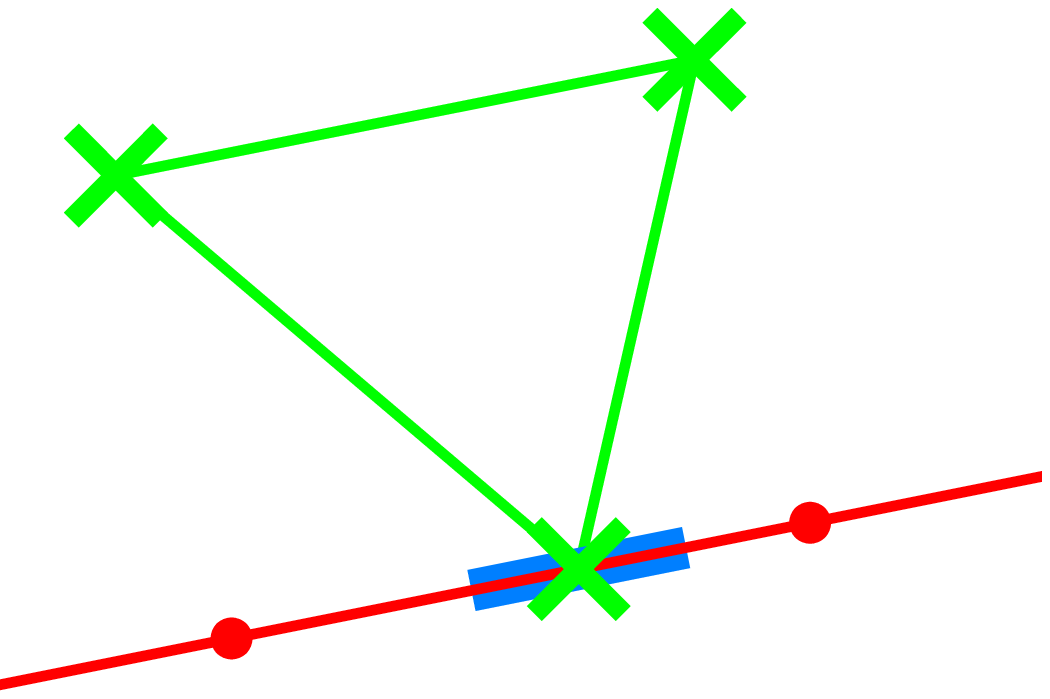}}}
}
\parbox{14.8cm}{
\parbox{6.0cm}{ 
\begin{fig} {\bf The camera collinearity ambiguity}. We can have arbitrarily many cameras
on the line. The point $P$ on the line can move. 
\end{fig}
}
\hspace{2.0cm}
\parbox{6.0cm}{ 
\begin{fig} {\bf The point collinearity ambiguity}. We can have arbitrarily many points
on the line. The camera on the line can move. 
\end{fig}
}
}

\parbox{14.8cm}{
\parbox{6.5cm}{\scalebox{0.60}{\includegraphics{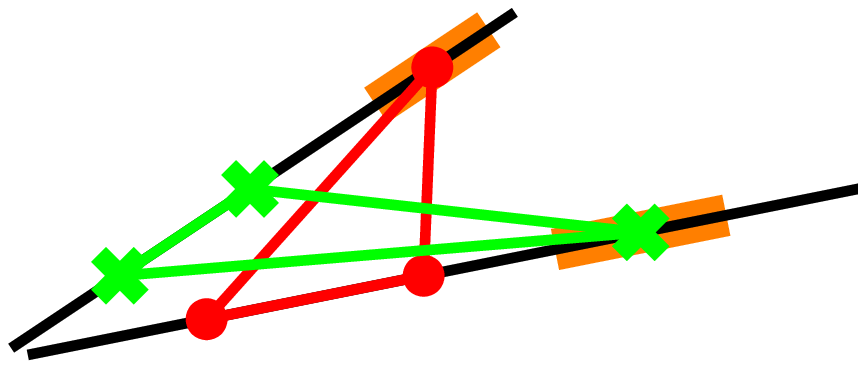}}}
\parbox{6.5cm}{\scalebox{0.60}{\includegraphics{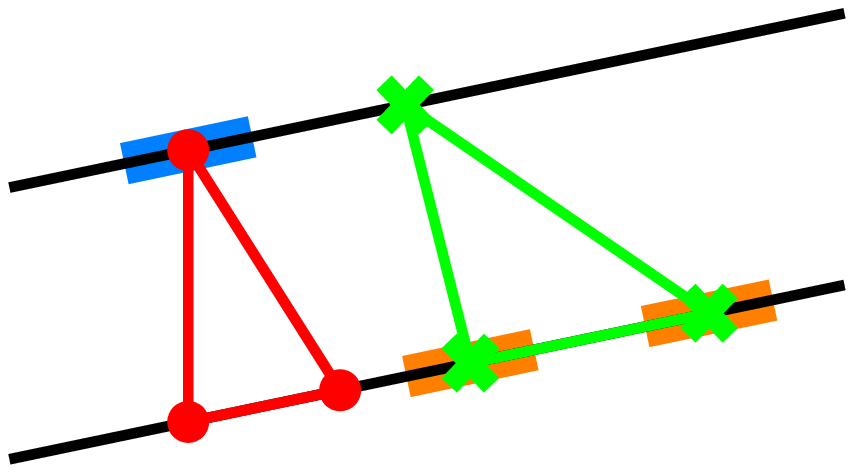}}}
}
\parbox{14.8cm}{
\parbox{6.0cm}{ \begin{fig}
{\bf A two line ambiguity I}. If cameras and points are contained in 
the union of two lines, it is possible that a camera point pair can 
be deformed without changing camera point angles. 
\end{fig}
}
\hspace{2.0cm}
\parbox{6.0cm}{ 
\begin{fig}
{\bf A two line ambiguity II}. This is an example where a pair of cameras and
a point can be deformed. The configuration is contained in the union of 
two parallel lines. 
\end{fig}
}
}
\vspace{1cm}

\parbox{14.8cm}{
\parbox{6.5cm}{\scalebox{0.60}{\includegraphics{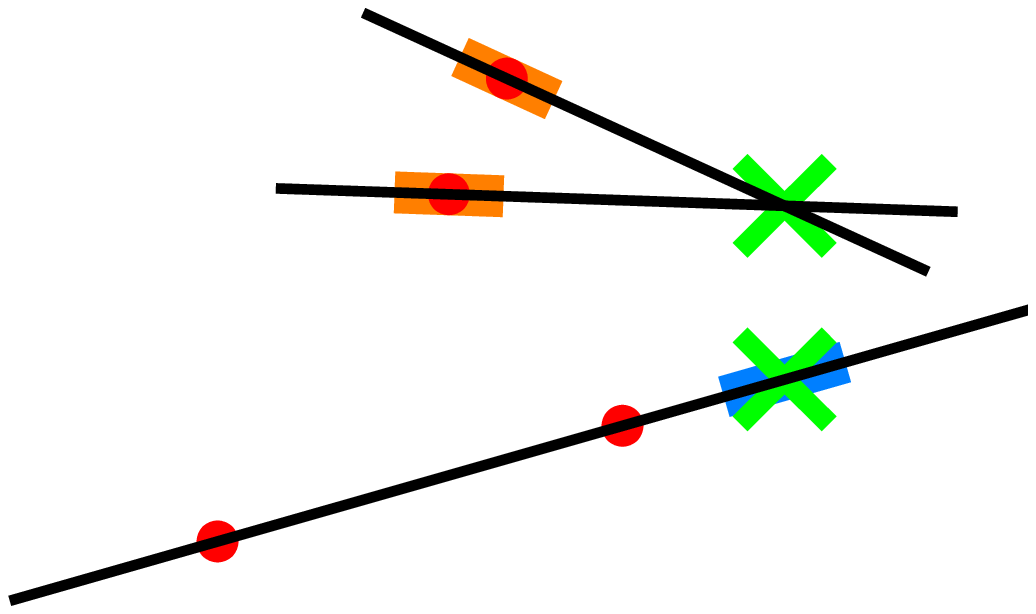}}}
\parbox{6.5cm}{\scalebox{0.60}{\includegraphics{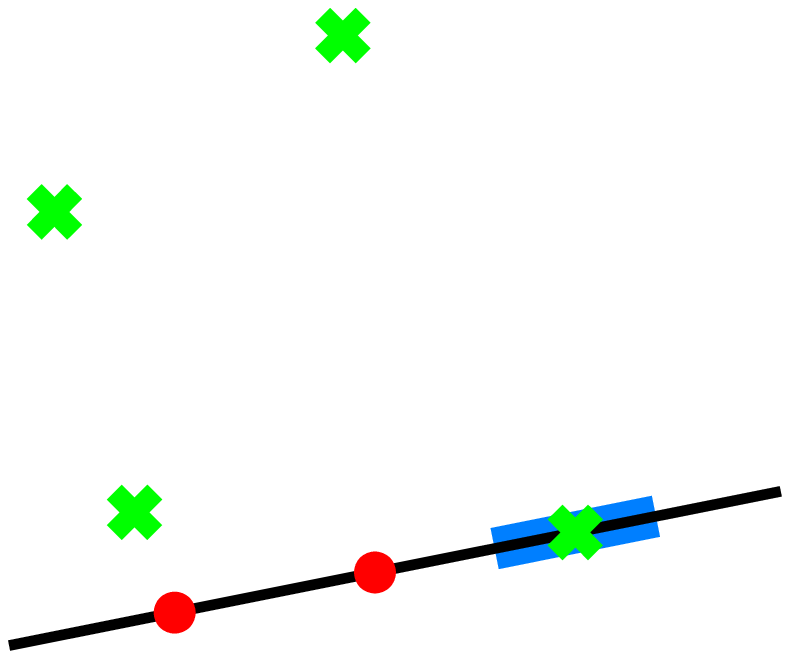}}}
\parbox{6cm}{
\begin{fig}
A two camera ambiguity, where one camera can move without changing image data. 
One can add arbitrarily many more points on the line containing the moving camera. 
Cameras and points do not need to be on the union of two lines. 
\end{fig}
}
\hspace{2cm}
\parbox{6cm}{
\begin{fig}
A two point ambiguity. One can add arbitrarily many cameras and
the union of the point camera set does not need to be on the union 
of two lines.
\end{fig}
}
}

\vspace{1cm}

Let's compare the two sides of the dimension formula in the oriented planar omni-directional
case $(d,f,g) = (2,2,3)$:  \\

\begin{tabular}{|c|c|c|c|c|} \hline
   Cameras &  Points        &  equations nm    &  unknowns 2(n+m)-3  &   unique ?  \\ \hline 
   m=1     &     n          &     n            &    2n-1             &   no, one camera ambiguities   \\
   m=2     &     n          &    2n            &    2n+1             &   no, two camera ambiguities  \\
   m=3     &    $n=2$       &     6            &     7               &   no, two point ambiguities  \\
   m=3     &    $n \geq 3$  &    3n            &    2n+3             &   yes, if no ambiguities\\
   m=4     &    $n \geq 3$  &    4n            &    2n+5             &   yes, if no ambiguities\\ \hline
\end{tabular}

\parbox{14.8cm}{
\parbox{7cm}{\scalebox{0.60}{\includegraphics{goodpoints/orientedomni2d.ps}}}
\parbox{7cm}{
\begin{fig}
In the plane $d=2$ with camera parameters $(f,g)=(2,3)$. The reconstruction region 
$dn+fm \leq (d-1) n m + g$ is given by $mn-2m-2n+3 <0$. The situation $(n,m)=(3,3)$ is 
the only borderline case. Also in all other cases $n,m \geq 3$ we have more or equal 
equations than unknowns and the reconstruction is unique if the conditions of the theorem 
are satisfied. 
\end{fig}
}
}

\section{Reconstruction for omni-cams in space}

For points $P_i =(x_i,y_i,z_i)$ and camera positions $Q_j =(a_j,b_j,c_j)$ in space,
the full system of equations for the unknown coordinates is 
nonlinear. However, we have already solved the problem in the plane and
all we need to deal with is another system of linear equations for the third coordinates 
$z_i$ and $c_j$. \\

If the slopes $n_{ij} = \cos(\phi_{ij})/\sin(\phi_{ij})$ are known, then
$$  c_i - z_j = n_{ij} r_{ij} \; ,  $$
where $r_{ij} = \sqrt{(x_i-a_j)^2 + (y_i-b_j)^2}$ are the distances
from $(x_i,y_i)$ to $(a_j,b_j)$. This leads to a system of 
equations for the additional unknowns $z_i$ and $c_j$. \\

\begin{center}
\parbox{7.5cm}{\scalebox{0.30}{\includegraphics{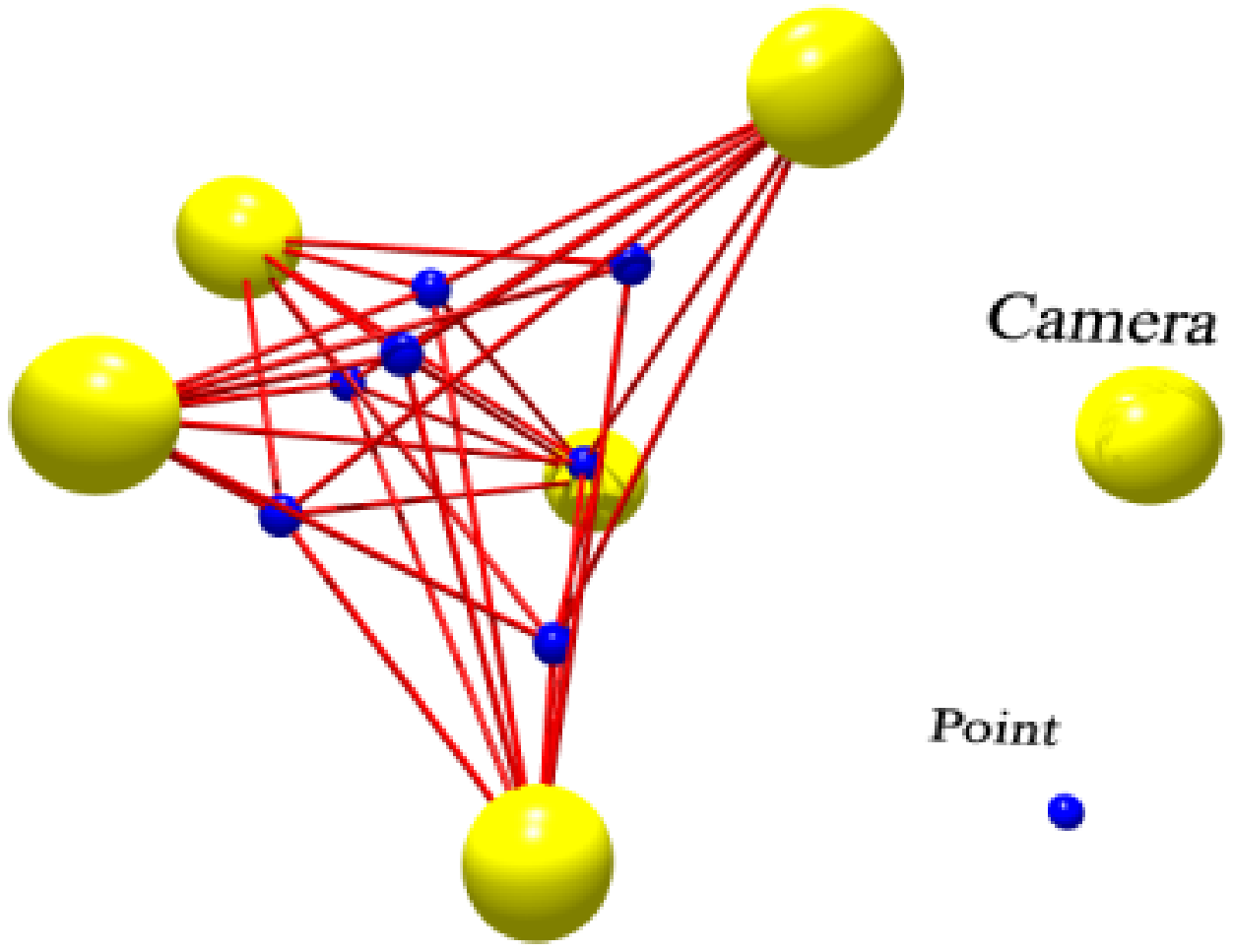}}}
\end{center}
\begin{fig}
The structure of motion problem for omni-directional cameras in space. We know the 
angles between cameras and points and want to reconstruct all the camera positions and
all the point positions up to a global rotation, translation and scale. 
\end{fig}

Here is the corresponding result of omni-directional camera reconstruction in space: 

\begin{thm}
The reconstruction of the scene and camera positions in three-dimensional
space has a unique solution if both the $xy$-projections of the point configurations 
as well as the $xy$-projection of the camera configurations are not collinear and
the union of point and camera projections are not contained in the union of two lines.
\end{thm}

\begin{proof}
We are led to the following linear system of equations
\begin{eqnarray}
 \cos(\theta_{ij}) (b_i - y_j) &=& \sin(\theta_{ij}) (a_i - x_j)  \\
 \cos(\phi_{ij})   (c_i - zj)  &=& \sin(\phi_{ij}) r_{ij}  \\  
 (x_1,y_1,z_1)                 &=& (0,0,0) \\
 x_2                           &=& 1 
\end{eqnarray}
for the unknown camera positions $(a_i,b_i,c_i)$ and scene points $(x_j,y_j,z_j)$. 
They are solved in two stages. First we solve $n m+3$ equations for the $2n+2m$
unknowns
\begin{eqnarray}
 \cos(\theta_{ij}) (b_i - y_j) &=& \sin(\theta_{ij}) (a_i - x_j)  \\
 (x_1,y_1)                     &=& (0,0) \\
 x_2                           &=& 1 
\end{eqnarray}
using a least square solution. If we rewrite this system of linear equations as 
$A x = b$, then the solution is $x_{min} = (A^T A)^{-1} A^T b$. \\

By the uniqueness theorem in two dimensions, this reconstruction is unique 
for the points $x_i,y_i,a_j,b_j$. Now we form 
$r_{ij} = \sqrt{(a_i-x_j)^2 + (b_i-y_j)^2}$ and solve the
$n m+1$ equations
\begin{eqnarray}
 \cos(\phi_{ij})   (c_i - z_j)  &=& \sin(\phi_{ij}) r_{ij}  \\  
 z_1                            &=& 0
\end{eqnarray}
for the additional $n+m$ unknowns. Also this is a least square problem. In case we have
two solutions, we have an entire line of solutions. This implies that we can find a deformation
$c_i(t),z_j(t)$ for which the angles $\phi_{ij}(t)$ stays constant. 
Because the $xy$-differences $r_{ij}$ of the points are known, these fixed angles assure
that the height differences $c_i - z_j$ between a camera $Q_i$ and a point $P_j$ is constant. 
But having $c_i-z_j$ and $c_k-z_j$ constant assures that $c_i-c_k$ is constant too and similarly
having $c_i-z_j$ and $c_i-z_k$ fixed assures that $z_j-z_k$ is fixed.  In other words, the only 
ambiguity is a common translation in the $z$ axes, which has been eliminated by assuming $z_1=0$.
The reconstruction is unique also in three dimensions. 
\end{proof}

{\bf Remarks.} \\
1) There is nothing special about taking the $xy$-plane to reduce the dimenson from $3$ to $2$. 
We can adjust the orientation of the
cameras arbitrarily. So, if 3 points are not collinear in space and three camera positions in 
space are not collinear and the camera-point set is not contained in the union of two lines, 
then a unique reconstruction is possible. Also, if four points define a tetrahedron of positive volume
and three camera positions are not on a line, then a unique reconstruction is possible. \\
2) The result also sheds some light on perspective cameras. Assume we take three pictures of
three points and if the camera orientation is identical for all three pictures, then we can 
reconstruct the point and the camera positions up to a scale and translation, if both points and
cameras are not collinear and the point camera set is not contained in the union of two lines. \\
3) If the union of the camera and point configurations is coplanar, then the ambiguity examples in 2D apply. 
If the camera point configurations are not coplanar, then two line ambiguity disappears so that
3 non-collinear camera points and 3 non-collinear scene points which are all not in a common plane
determine the situation. \\
4) In real world applications, we don't see a point at all times, because objects
sometimes obscure other objects. Let's call $\cal{P}$ the set of
points which we can see for some time interval during the movie and let $\cal{G}$
be the set of pairs $(i,j)$, for which we can observe the point $P_j$ at time $t_i$.
The pair $(\cal{P},\cal{G})$ is a graph.
If the camera can see a point for an average fraction $r= 2(n+m)/(nm+3)$ of times, then the system 
is expected to have a least square solution. For example, for a movie with $m=10$ frames observing $n=10$
points, then we need to see the points for a fraction of $60/200=0.3$ that is 
for 30 percent of the times. For $m=100$ frames observing $n=1000$ points,
we need to see an average point less than 2 percent of the time in order to do 
the reconstruction. A concrete reconstruction would use as many movie frames as possible
for a time interval $[a_1,b_1]$, then make a new reconstruction for an other time interval
$[a_2,b_2]$ etc, where $[a_i,b_i]$ are overlapping intervals on the time axes. If we look on 
each interval at the set of points which are visible at all times and for these points, the
conditions of the theorem are satisfied, then the reconstruction is unique. \\

The number of cameras can not be reduced in the plane but it can be reduced in space. 
Two cameras and two points are enough in space in general:

\begin{propo}
For $m=2$ oriented omni-directional cameras in space observing
$n \geq 2$ points, the point-camera configuration is determined up to scale and Euclidean
transformation if the point-camera configuration is not coplanar. The situation is
ambiguous in two dimensions with 2 cameras for an arbitrary number $n$ of points.
\end{propo}

\parbox{14.8cm}{
\parbox{6.5cm}{\scalebox{0.60}{\includegraphics{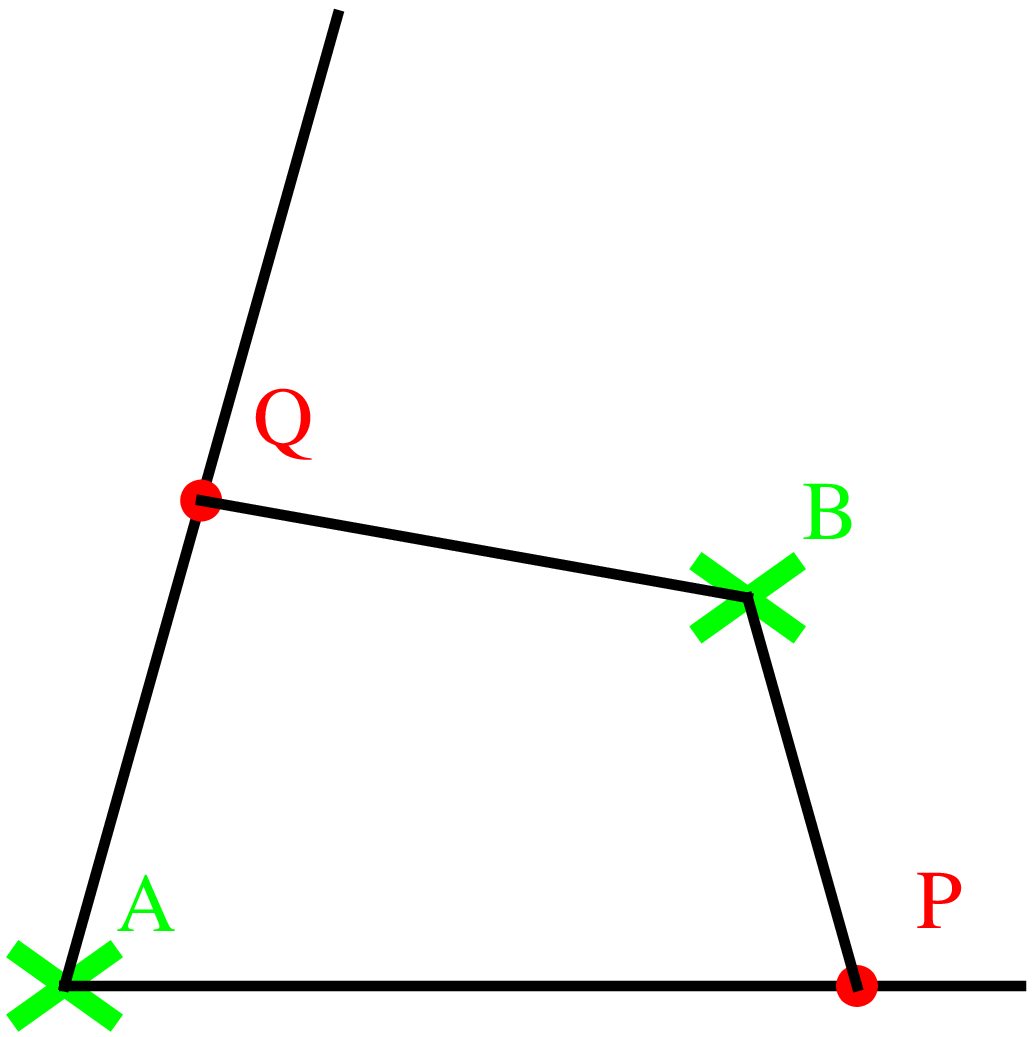}}}
\parbox{6.5cm}{
\begin{fig}
Two oriented omni-directional cameras and two points in the plane. The angles between cameras and points do not
determine the configuration. Arbitrary many points can be added. In three dimensions however,
two points $P,Q$ and two cameras $A,B$ allow a reconstruction 
because the directions $PA,PB,QA,QB$ of the tetrahedron sides
determines the shape of the tetrahedron up to a dilation and a Euclidean transformation.
The 4 points $A,B,C,D$ need to be non-coplanar.
\end{fig}}
}

\section{Structure from motion with moving bodies}

We assume now that a omni-directional camera moves through a scene, in which the bodies themselves 
can change location with time. Examples are a car moving in a traffic lane, a team football 
players moving on a football field or the earth observing the planets moving around it. \cite{RLSP06a} \\

The reconstruction needs more work in this case, but the problem remains linear if we make a Taylor
expansion of each point path. Again the reconstruction
is ambiguous if we do not fix one body because the entire scene as well as the camera could move
with constant speed and provide alternative solutions. This ambiguity is removed by assuming one
point in the scene to have zero velocity. \\

Assume first that we have a linear motion $P_i(t) = P_i + t P_i'$ of the points. There are 
now twice as many variables for the points because both positions and velocities are unknown. 

\begin{eqnarray}
 \cos(\theta_{ij}) (b_i - y_j - t_j y_j') &=& \sin(\theta_{ij}) (a_i - x_j - t_j x_j')  \\
 \cos(\phi_{ij})   (c_i - zj - t_j z_j')  &=& \sin(\phi_{ij}) r_{ij}  \\  
 (x_1,y_1,z_1)                 &=& (0,0,0) \\
 (x_1',y_1',z_1')              &=& (0,0,0) \\
 x_2                           &=& 1 
\end{eqnarray}
for the unknown camera positions $(a_i,b_i,c_i)$ and scene points $(x_j,y_j,z_j)$
and scene point velocities $(x_j',y_j',z_j')$. The nonlinear system can again be 
reduce to  two linear systems.  \\

This can be generalized to the case when we have a finite Taylor expansion.
$$  P_i(t) = \sum_{l=0}^k P_i^{(l)} t^l $$ 
for every point. We still have
$n m$ equations and a global $g$ dimensional symmetry but now $3nk+3mf$ 
unknown parameters. 
If the motion of every point in the scene is described with a Taylor expansion 
of the order $k$, then the structure from motion inequality is 
$$ d n (k+1) + m f \leq n m (d-1) + g  \; . $$

With moving bodies, there can be even more situations, where the motion can not be reconstructed:
take an example with arbitrarily many points, but where two points $P_1(t),P_2(t)$ form a line
with the camera position $r(t)$ at all times.
In that case, we are not able to determine the distance between these two points
because the points are on top of each other on the movie.

\parbox{14.8cm}{
\parbox{6.5cm}{\scalebox{0.60}{\includegraphics{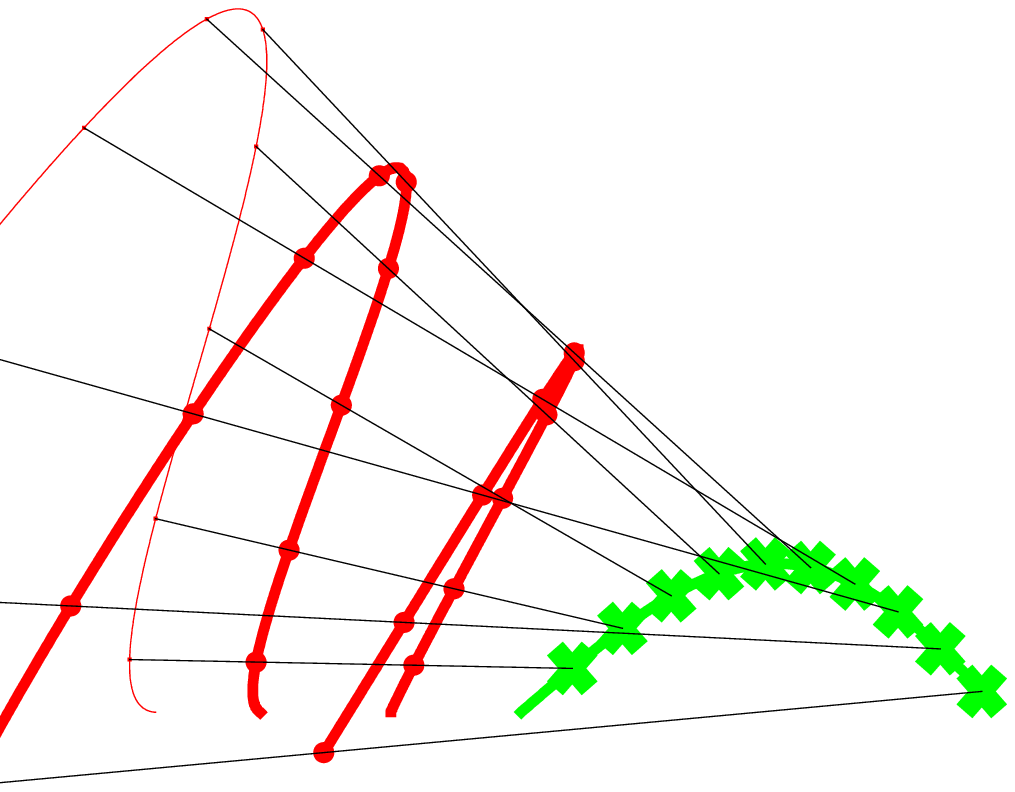}}}
\parbox{6.5cm}{ 
\begin{fig}
{\bf The hidden point ambiguity}. 
If one of the points moves so that it always
is behind an other point, we have no information to determine
the distance of the second point from the first. There are 
several point motions which produce the same angular data. 
\end{fig}
}
}

Instead of a Taylor expansion of the moving bodies, one could also do a 
Fourier decomposition of the motion
$$  P_i(t) = P_i + \sum_{l=1}^k A_{il} \cos(l t) + B_{il} \sin(l t) $$
and solve for the unknowns $P_i,A_{il},B_{il}$.  Again we are lead to a system of
linear equations for which we can look for least square solutions. The dimension formula
would be 
$$ dn (2k) + m f \leq n m (d-1) + g  $$

\section{Non-oriented omni cameras}

We have assumed that the omni-directional camera has a fixed "up" direction and points
in a fixed direction like "north" at all times. If an omni-directional camera moves
in a car, then it can turn. Additionally, the camera could rotate arbitrarily during the motion.
This is described by a curve $\theta(t),\phi(t))$. For work on ego-motion estimates in omni-directional
view, see \cite{GandhiT05}. Because the additional unknowns $\theta_j = \theta(t_j), \phi_j = \phi(t_j)$ 
enter in a nonlinear way into the equations, it is better to deal with this problem separately. 
Here is a {\bf mean motion algorithm} for computing the camera orientation motion. We assume that the 
camera motion itself is {\bf adiabatic}, meaning that the angular motion of the camera is small
compared with the frame rate. A camera built into an plane or a car would produce 
an adiabatic camera motion. A non-adiabatic example would be an omni-directional camera built into 
a tennis ball which has been hit with a spin. \\

\noindent
1) Compute the angular velocities of all points $P_i$ at the times $t_j$ with  
$$  \omega_i(t_j) = \frac{\theta_{i,j+1}-\theta_{i,j}}{t_{j+1}-t_j}, 
    \eta_i(t_j)   = \frac{\phi_{i,j+1}-\phi_{i,j}}{t_{j+1}-t_j} \; . $$
This is called the motion field.  \\
2) Find the average angular velocities
$$  \omega(t_j) = \frac{1}{n} \sum_{i=1}^n \omega_i(t_j),  \;\;\; 
    \eta(t_j) = \frac{1}{n} \sum_{i=1}^n \cos(\omega_i(t_j)) \eta_i(t_j)  \; . $$ 
This is the mean motion of the optical flow. \\
3) Produce a first approximation of the camera orientation motion
$$  \theta(t_j) = \sum_{l=1}^{j-1} \omega(t_l), \;\;\;
    \phi(t_j) = \sum_{k=1}^{j-1} \eta(t_k) \; . $$
4) Now fine tune the variables $\phi(t_j),\theta_j$ individually to minimize the 
least square solution for all $j$.  \\

Step 4) can be avoided if points are sufficiently far away from the camera. 
To compare it with our own vision capabilities: if we turn our head, 
then we do steps 1-3: we estimate the average speed with which the 
points around us move. This gives us an indication how the head moves. \\

{\bf Remarks}.
1) The situation with variable camera orientation could be
put into the framework of the moving bodies. This has the advantage that
the system of equations is still linear. The disadvantage is an explosion of 
the number of unknown variables. \\
2) A further refinement of the algorithm to first filter out points which 
are further away and only average the mean motion of those points. 
A rough filter is to discard points which move with large velocity. 
See \cite{GandhiT05} for a Bayesian approach. See also \cite{ThrunBurgardFox}.  \\

The problem of ambiguities in the case of unknown camera rotations is more 
complicated also because of the nonlinearity of the problem. Let's look at the 
dimensions in the planar case, where each camera located at $Q_j=(a_j,b_j)$ 
is turned by an angle $\theta_j$. We have $n m$ equations 
$$  \cos(\theta_{ij}-\theta_j) (b_i - y_j) = \sin(\theta_{ij}-\theta_j) (a_i - x_j)      $$
with $(x_1-x_2)^2 + (y_1-y_2)^2=1$ for the $2n+2m-3+m$ unknowns 
$$ \{ x_i,y_i \}_{i=2}^{n}, \{a_j,b_j \}_{j=1}^m,\{\theta_j\}_{j=1}^m \; . $$

\begin{center}
\parbox{15.5cm}{
\parbox{7.5cm}{\scalebox{0.45}{\includegraphics{goodpoints/orientedomni2d.ps}}}
\parbox{7.5cm}{\scalebox{0.45}{\includegraphics{goodpoints/orientedomni3d.ps}}}
}
\end{center}
\begin{center}
\parbox{15.5cm}{
\parbox{7.5cm}{Oriented Omni  \\ (d,f,g) = (2,2,3)}
\parbox{7.5cm}{Oriented Omni  \\ (d,f,g) = (3,3,4)}
}
\end{center}
\begin{fig}
The forbidden region in the $(n,m)$ plane for oriented omni-directional cameras.
\end{fig}

{\bf Remarks}.  \\
1) It seems unexplored, under which conditions the construction is 
unique for non-oriented omni-cameras. Due to the nonlinearity 
of the problem, this is certainly is not as simple as in the 
oriented case. \\

2) For omni-directional cameras in space which all point in the same direction but turn around this axis, 
the dimension analysis is the same. We can first compute the first two coordinates and then the third coordinate.
When going to the affine limit, these numbers apply to camera pictures for which we know one direction.
This is realistic because on earth, we always have a gravitational direction. So, if we know
the direction of the projection of the $z$ axis onto the picture, then we can reconstruct with $3$ pictures
and $6$ points.

\section{Experiments}

How does the reconstructed scene depend on measurement or computation errors? 
How sensible is the least square solution on the entries of $A$? The error 
depends on the volume $\sqrt{{\rm det}(A^TA)}$ of the parallelepiped spanned by the columns of 
$A$ which form a basis in the image of $A$. Because this parallelepiped has positive volume
in a non-ambiguous situation, we have:

\begin{coro}
The maximal error $\epsilon$ of the reconstruction depends 
linearly on the error $\delta$ of the angles $\theta_{ij}$ and $\phi_{ij}$ for small $\delta$. 
There exists a constant $C$ such that $|\epsilon(\delta)| \leq C |\delta|$.
\end{coro}

\begin{proof}
The reconstruction problem is a least square problem $Ax=b$ which has the solution 
$x_* = (A^T A)^{-1} A^T b$. Without error, 
there exists a unique solution. In general, $A$ has no kernel if we are not in an ambiguous situation.
The error constant depends on the maximal entries of $A$ and 
$C = 1/{\rm det}(A^T A^{-1})$ which are both finite if $A$ has no kernel. 
\end{proof}

Empirically, we confirm that the maximal error is of the order of $\delta$. We only made experiments with 
synthetic but random data and see that the constant $C$ is quite small.  
The computer generated random points, photographs them with an omni-directional cameras 
from different locations and reconstructs the point locations from the angle data. 
The maximal error is expected to grow for a larger number of points because the length of the error
vector grows with the dimension. A random vector $\vec{\delta} = (\delta_1, \dots ,\delta_n)$ with 
$\delta_i$ of the order $\delta$ has length of the order $\sqrt{n} \delta$.  \\

\parbox{14.8cm}{
\parbox{4.8cm}{\scalebox{0.40}{\includegraphics{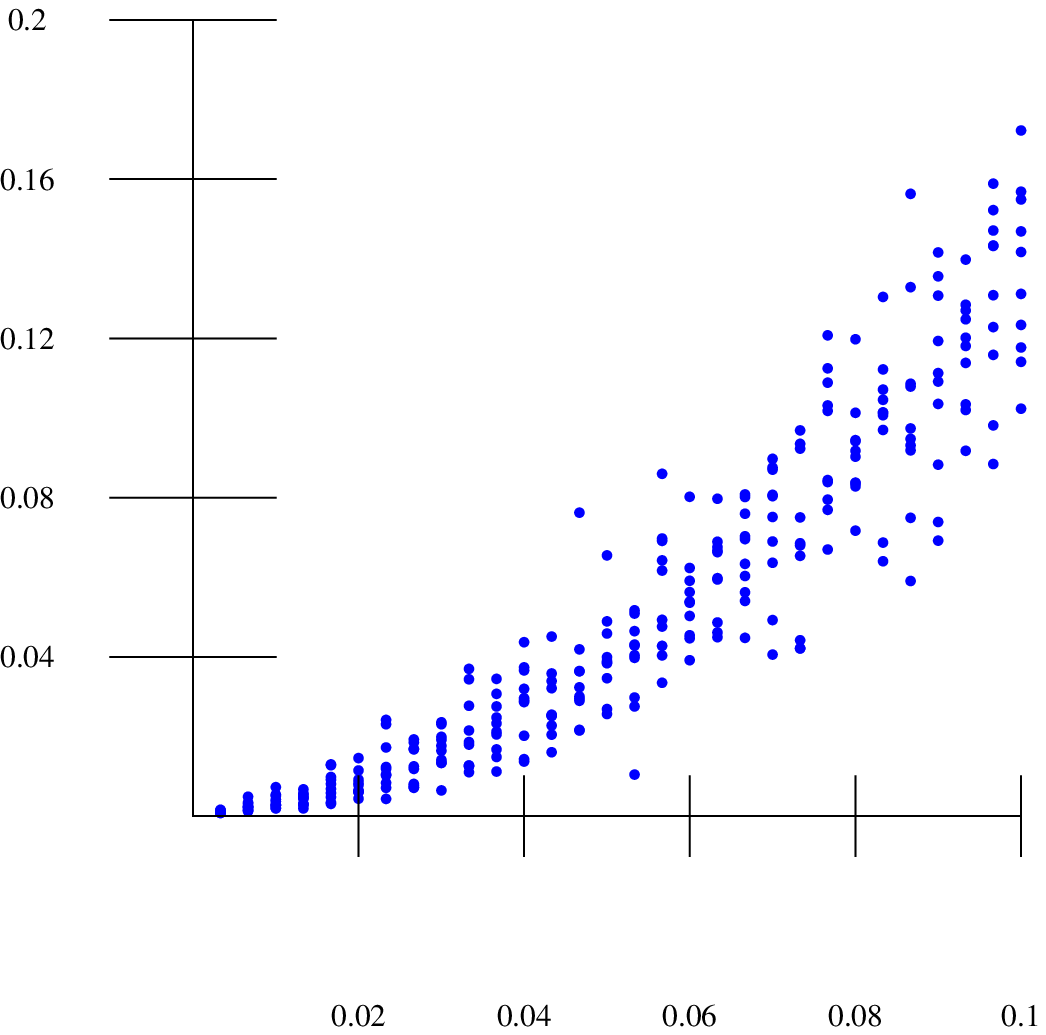}}}
\parbox{4.8cm}{\scalebox{0.40}{\includegraphics{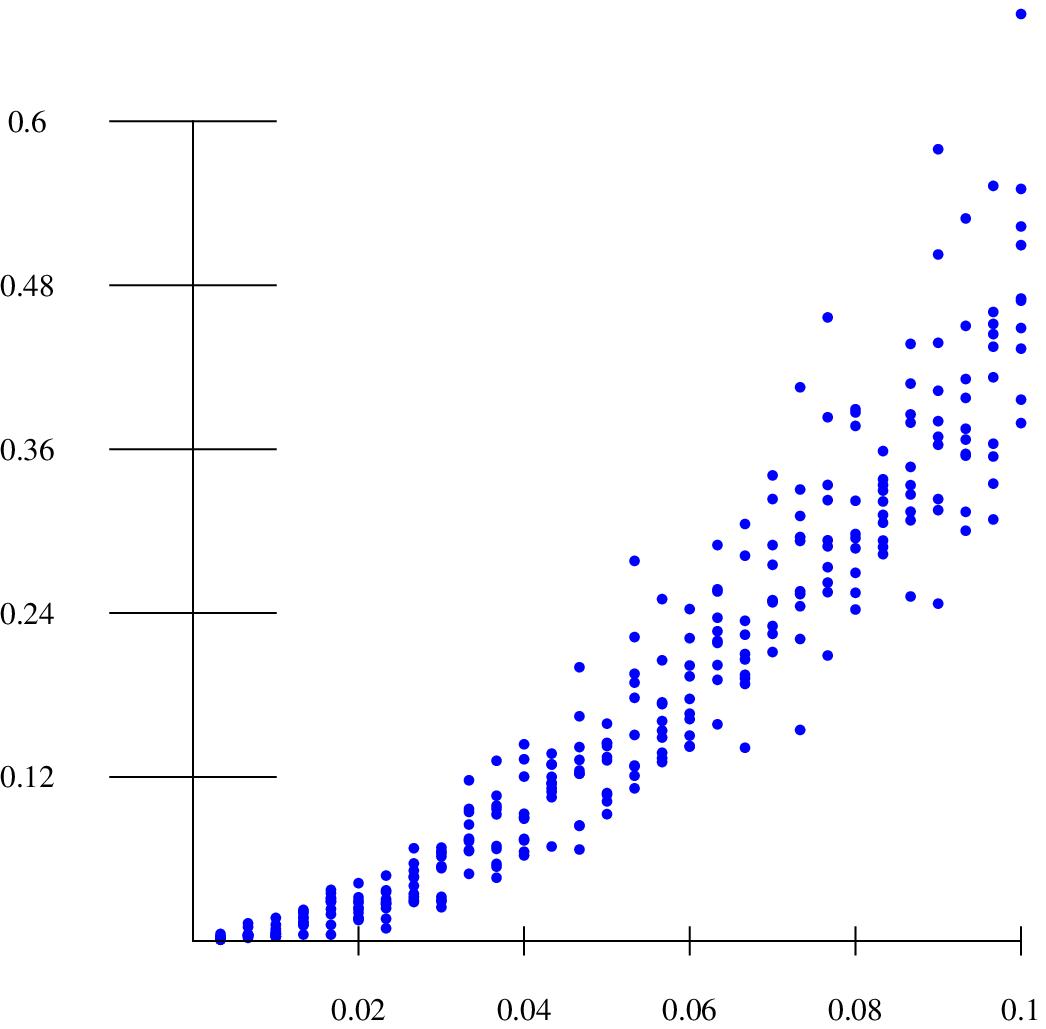}}}
\parbox{4.8cm}{\scalebox{0.40}{\includegraphics{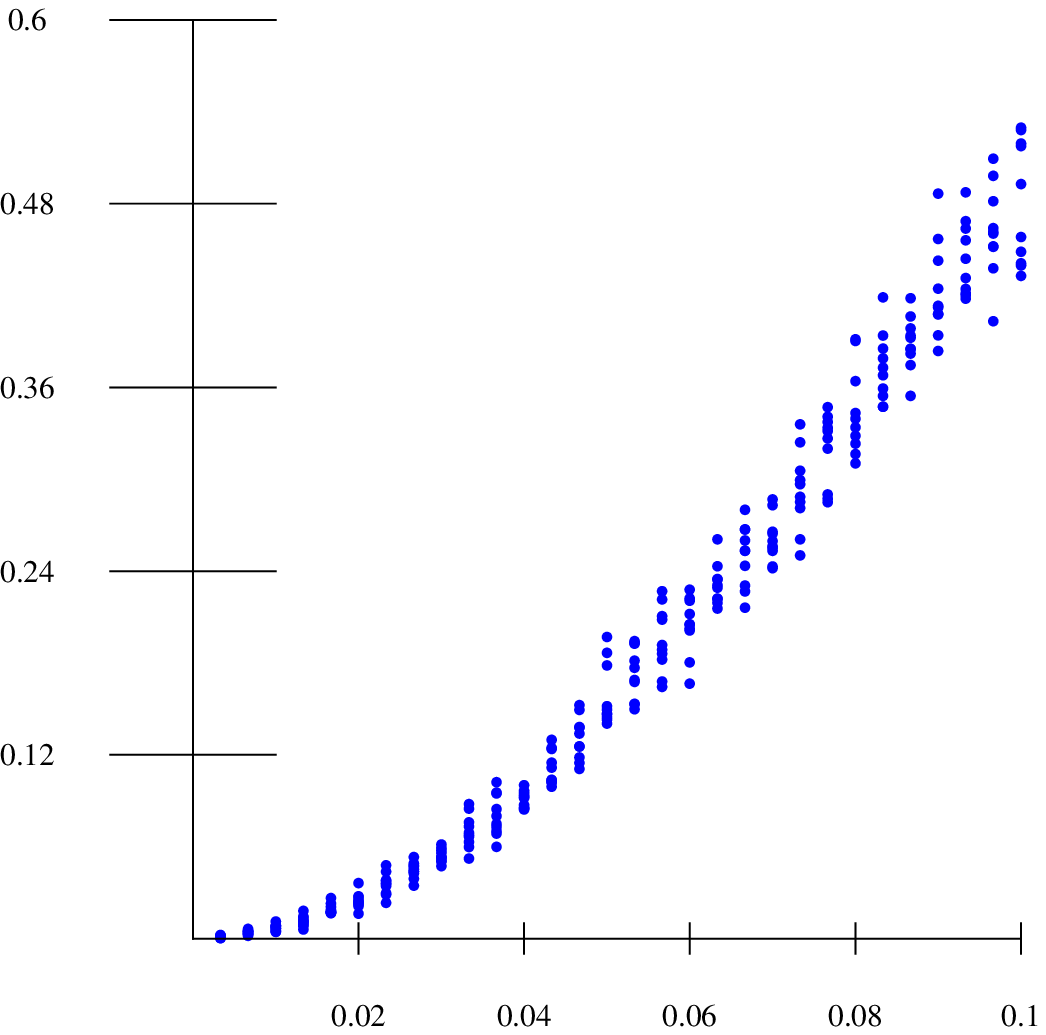}}}
}
\begin{fig}
The maximal error of the reconstruction depends on the error size.
Experiments with $(m,n) = (10,50), (50,10), (50,50)$. 
The maximal reconstruction error depends  in a linear way
on the error added to each coordinate. For every of the 30 
$\epsilon$ values between $0$ and $0.1$, we do 
the reconstruction for 10 randomly chosen camera-point configurations and plot for 
each experiment the {\bf maximal} deviation among all the coordinates of all the 
reconstructed points. 
\end{fig}

If $\{ \epsilon_{i} \}_{i=1}^{n+m}$ are the distances from the $n+m$ original points $P_i$ and 
cameras $Q_j$ to the reconstructed points and cameras, we have computed the 
maximal error ${\rm max} |\epsilon_i|$. The {\bf mean absolute errors} $\frac{1}{n+m} \sum_i |\epsilon_i|$
the {\bf mean errors} $\frac{1}{n+m} \sum_i \epsilon_i$ as well as the {\bf root mean errors}
$(\frac{1}{n+m} \sum_i \epsilon_i^2)^{1/2}$ are much smaller. \\

\begin{center}
\parbox{6.2cm}{\scalebox{0.40}{\includegraphics{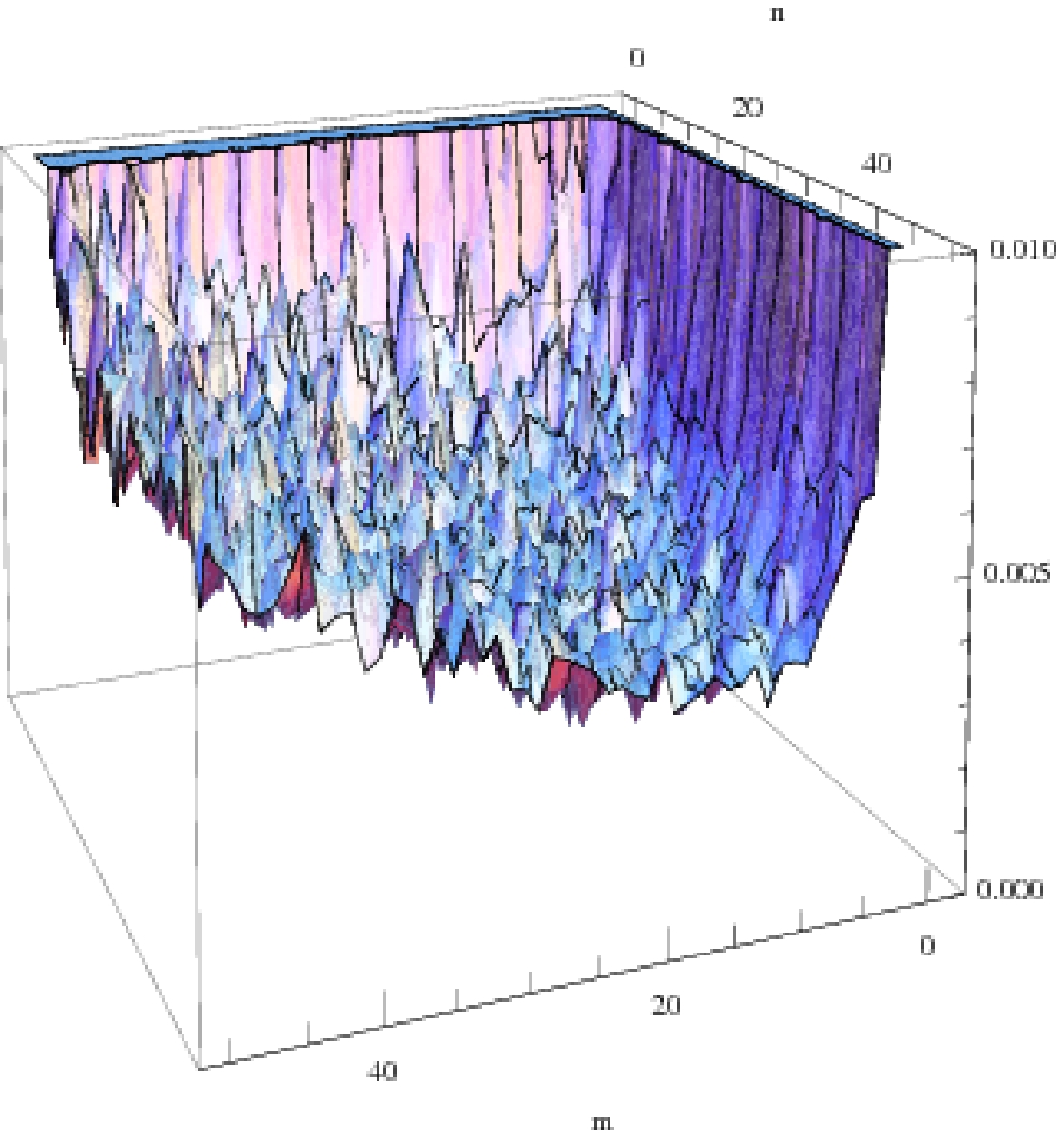}}}
\end{center}
\begin{fig}
The maximal reconstruction error depending on the number $1 \leq m \leq 50$ of cameras and the 
number $1 \leq n \leq 50$ of points. The picture shows the same graph from two sides. 
We fixed $\epsilon =0.01$. 
For each $m$ and $n$, the reconstruction was done for 30 different camera-point worlds. 
The graph shows the average over these 30 samples.
For each of the $50 \times 50 \times 30$ experiments, each camera and each point is 
displaced with an independent error of amplitude 
$[-\epsilon,\epsilon]$. The experiment indicates that after a sharp decay for small $n,m$ where we have
ambiguities, the error grows linearly at most and the change is larger for more cameras than
more points. 
\end{fig}

{\bf Remarks:} \\
1. The average error decreases like $1/n$ because the maximal error is essentially
independent of $n$. \\
2. From the practical points of view, we are also interested in how much 
aberration we see when the reconstructed scene is filmed again.
Geometrically, the least square solution $x_{*}$ of the system $Ax = b$ has the property that 
$Ax_{*}$ is the point in the image of $A$ which is closest to $b$. 
If the reconstructed scene is filmed again, then even with some
errors, the camera sees a similar scene. Because  $A (A^T A)^{-1} A^T$ is a projection
onto the image of $A$, this projected error is of the order $1$. In other words, we see
no larger errors than the actual errors. 

For refined error estimates for least square solutions see \cite{Wei,GolubVanLoan}. \\

\vspace{12pt}
\bibliographystyle{plain}
\bibliography{3dimage}
\end{document}